\documentclass{article}

\usepackage{microtype}
\usepackage{graphicx}
\usepackage{booktabs} 

\usepackage{hyperref}

\usepackage{setspace}

\usepackage[accepted]{icml2024}

\usepackage{amsmath}
\usepackage{amssymb}
\usepackage{mathtools}
\usepackage{subcaption}
\usepackage{amsthm}
\usepackage{tcolorbox}

\usepackage[capitalize,noabbrev]{cleveref}

\theoremstyle{plain}
\newtheorem{theorem}{Theorem}[section]

\newtheorem{lemma}[theorem]{Lemma}

\theoremstyle{definition}

\theoremstyle{remark}

\usepackage[textsize=tiny]{todonotes}
\usepackage{multirow}
\usepackage{graphicx}
\usepackage{caption}
\usepackage{longtable}

\usepackage{algorithm}
\usepackage{algorithmic}

\icmltitlerunning{Differentiable Combinatorial Scheduling at Scale}

\newcommand{\fixme}[1]{\textcolor{red}{#1}}
\newcommand{\zz}[1]{\textcolor{blue}{\small [zz: #1]}}

\begin{document}

\twocolumn[
\icmltitle{Differentiable Combinatorial Scheduling at Scale}



\icmlsetsymbol{equal}{*}
\icmlsetsymbol{correspond}{}

\begin{icmlauthorlist}
\icmlauthor{Mingju Liu}{equal,umd}
\icmlauthor{Yingjie Li}{equal,correspond,umd}
\icmlauthor{Jiaqi Yin}{umd}
\icmlauthor{Zhiru Zhang}{cornell}
\icmlauthor{Cunxi Yu}{umd,correspond}
\end{icmlauthorlist}

\icmlaffiliation{umd}{University of Maryland, College Park}
\icmlaffiliation{cornell}{Cornell University}

\icmlcorrespondingauthor{Yingjie Li}{yingjiel@umd.edu}

\icmlkeywords{Machine Learning, ICML}

\vskip 0.3in
]

\printAffiliationsAndNotice{\icmlEqualContribution} %


\begin{abstract}
This paper addresses the complex issue of resource-constrained scheduling, an NP-hard problem that spans critical areas including chip design and high-performance computing. Traditional scheduling methods often stumble over scalability and applicability challenges. We propose a novel approach using a differentiable combinatorial scheduling framework, utilizing Gumbel-Softmax differentiable sampling technique. This new technical allows for a fully differentiable formulation of linear programming (LP) based scheduling, extending its application to a broader range of LP formulations. To encode inequality constraints for scheduling tasks, we introduce \textit{constrained Gumbel Trick}, which adeptly encodes arbitrary inequality constraints. Consequently, our method facilitates an efficient and scalable 
scheduling via gradient descent without the need for training data. Comparative evaluations on both synthetic and real-world benchmarks highlight our capability to significantly improve the optimization efficiency of scheduling, surpassing state-of-the-art solutions offered by commercial and open-source solvers such as CPLEX, Gurobi, and CP-SAT in the majority of the designs.
\end{abstract}

\section{Introduction}
\label{sec:introduction}

Nowadays, the computer-aided scheduling techniques have been widely used in various tasks, such as computing \cite{cong2006efficient, floudas2005mixed, davis2011survey, dhall1978real,steiner2022olla,kathail2020xilinx,babu2021xls}, operations research \cite{kolisch1997psplib, laborie2018ibm, hartmann2022updated}, automated systems \cite{booth2016constraint, booth2016mixed, schmitt2016perception, tran2017robots}, transportation \cite{cappart2017rescheduling, gedik2017constraint, kinable2016optimization}.  Scheduling plays a crucial role in optimizing time, resources, and productivity, leading to better outcomes and improved efficiency. For example, in the context of computing systems, scheduling is a critical step in computing systems, ensuring optimal performance in hardware synthesis by efficiently allocating resources and timing, and in compilers by determining the sequence of operations to optimize code execution and resource usage.  

However, resource- or time-constrained scheduling is a known \textit{NP-hard} problem. Despite an extensive body of prior research and development on either exact or heuristic-based scheduling methods, 
contemporary scheduling approaches still have major limitations: 

(1) \textbf{Unfavorable speed-quality trade-off:} 
Many constrained scheduling problems can be solved exactly using integer linear programming (ILP) \cite{hwang1991formal,floudas2005mixed,steiner2022olla,yin2022exact}, satisfiability (SAT)~\cite{steiner2010evaluation, zhang2004sat, coelho2011multi}, or {constraint programming} (CP) formulations~\cite{christofides1987project,laborie2018ibm,baptiste2001constraint,cesta2002constraint}. 
%
However, these approaches suffer from limited scalability. Conversely, popular heuristic methods \cite{ahn2020ordering,paulin1989force,graham1969bounds,blum2003metaheuristics,brucker1998branch} often yield suboptimal results while achieving feasible run times. Notably, a heuristic method based on system of difference constraints (SDC) 
 provides an efficient formulation to encode a rich set of scheduling constraints in SDC and expresses the optimization objective in a linear function that can be solved as an LP problem~\cite{cong2006efficient, dai2018scalable}. 

(2) \textbf{Insufficient utilization of modern parallel computing devices:} Existing scheduling algorithms and solvers are primarily designed for single-threaded CPU execution and are unable to exploit modern parallel computing devices like GPUs \cite{sanders2010cuda} and TPUs \cite{jouppi2017datacenter}. 


Recently, machine learning (ML) has been used for combinatorial scheduling for compiler and hardware synthesis to improve its runtime efficiency and explore the expanded decision space~\cite{bengio2021machine, yu2018developing,yu2019painting, neto2022end, wu2023gamora}. There are mainly two categories: \textit{imitation learning}~\cite{baltean2018strong, gasse2019exact, gagrani2022neural, wang2023cardinality}, where the policy is learned through supervised targets while suffering from difficult data collection and poor model generalizability; \textit{reinforcement learning}~\cite{mascia2014grammar, karapetyan2017markov, chen2019RLS, yin2023respect, yin2023accelerating,yu2020flowtune,neto2022end}, where the policy is learned from the rewards and potential to outperform the current policy with new discoveries while suffering from limited problem scalability and significant runtime overhead. 


In this work, we introduce a scalable approach to differentiable combinatorial scheduling based on SDC formulations 
employing Gumbel-Softmax \cite{jang2016categorical} for the differentiation of scheduling variables and crafting constraints as differentiable distributions for variable discretization. In contrast to existing learning-based approaches, this allows for the customization of objective functions, as well as models the optimization problem of scheduling as a stochastic optimization problem that can be optimized without training and labeled data collection. 
As a result, our approach introduces an auto-differentiation process for solving combinatorial scheduling without model training. 
This new approach distinguishes itself from conventional methods by its ability to scale global optimization through parallel computing resources. Moreover, the proposed technique seamlessly integrates with existing ML frameworks like PyTorch, ensuring fast and practical implementation. Our experimental results demonstrate significant improvements in optimization efficiency over state-of-the-art (SOTA) methods solved with commercial solvers CPLEX \cite{cplex2009v12}, Gurobi \cite{gurobi} and open-source CP-SAT solver \cite{cpsatlp, perron2023cp}. Our experimental setups and implementations are available at \url{https://github.com/Yu-Maryland/Differentiable_Scheduler_ICML24}. 

\section{Preliminary}
\label{sec:preliminary}

\subsection{Scheduling and Problem Formulation}
\label{subsec:schedule}


Scheduling is one of the most extensively studied combinatorial problems with a wide range of real-world applications. This work focuses on scheduling a dataflow graph, with the input represented as a directed acyclic graph (DAG) \(G(V,E)\). 
In the domain of computing systems, these graphs consist of nodes \(V\), representing tasks that execute specific computations such as arithmetic, logical operations, or ML operators. The edges \(E\) represent the flow of data between these nodes. Additional cost metrics can be associated with the nodes and/or edges of the graph in the form of weights. Moreover, a set of scheduling constraints, such as timing constraints and resource constraints, are often specified as part of the formulation, depending on the target scheduling problem. The goal of the optimization is to generate a schedule $S = s_0,s_1,...s_i$, $i \leq |V|$, where $s_i$ represents the scheduled stage of node $v_i$, in order to satisfy the given constraints while minimizing or maximizing an objective which is a function of $S$. 



The targeted scheduling in this work is defined as follows: {\textbf{Given} a DAG $G(V,E)$, where $V$ is the list of nodes to be scheduled, each associated with a per-node resource cost, and $E$ are weighted edges capturing dependency constraints and edge costs. Latency $L$ is the time-to-completion for the entire graph, representing the time between the initiation and completion of the computational task captured by the DAG. The \textbf{objective} is to optimize the schedule w.r.t the dependency constraints under a given latency $L$ while minimizing the cost. In other words, we are solving a \underline{latency-constrained min-resource scheduling}.} 


\paragraph{System of Difference Constraint (SDC) --} An SDC is a system of difference constraints in the integer difference form, denoted as $x_i-x_j\leq c_{ij}$, where $c_{ij}$ is an integer constant, and $x_i$ and $x_j$ are discrete variables. SDC scheduling has been deployed in multiple commercial and open-source high-level synthesis (HLS) tools, such as AMD Xilinx Vivado/Vitis HLS \cite{kathail2020xilinx,cong2011high} 
and Google XLS \cite{babu2021xls}. An SDC is feasible if there exists a solution that satisfies all inequalities in the system. Due to the restrictive form of these constraints, the underlying constraint matrix of SDC is totally unimodular~\cite{camion-tum-ams1965}, enabling the problem (feasibility checks or optimization) to be solvable in polynomial time with LP solving while ensuring integral solutions. These constraints can be incorporated with a linear objective to formulate an optimization problem, which is leveraged in this work to handle the dependency constraints (Section \ref{sec:approach}). 


\begin{figure}[h]
\includegraphics[width=0.95\linewidth]{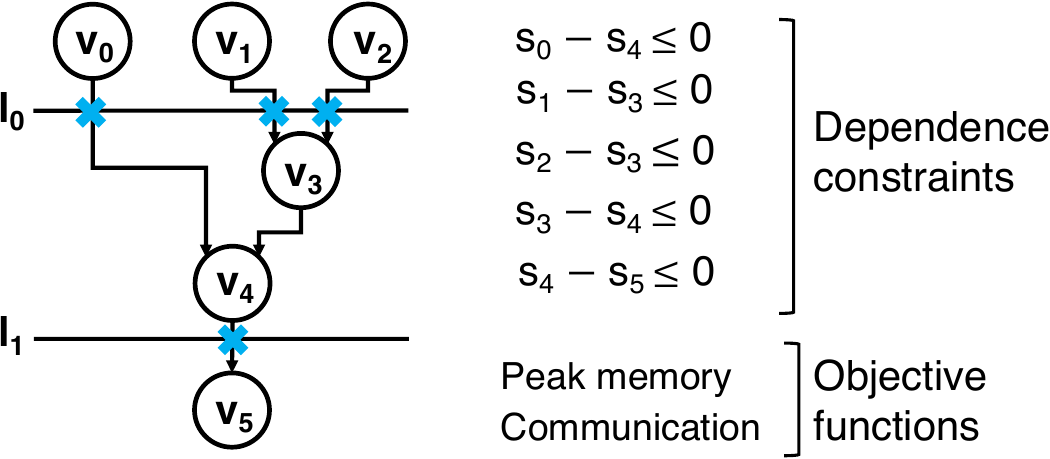}
\centering
\caption{Example of SDC-based scheduling --- (left) A DFG with two schedule stages $l_{0}$ and $l_{1}$ with latency $L=3$; (right) Dependence constraints and objective functions including peak memory minimization and inter-stage communication minimization (the blue crosses)} 
\label{fig:sdc-example}
\end{figure}

We illustrate the SDC-based scheduling formulation with a simple data flow graph (DFG) in Figure~\ref{fig:sdc-example}. To manage the \underline{dependencies}, SDC establishes a difference constraint for each data edge from operation $i$ to operation $j$ within the DFG, denoted as $s_i - s_j \leq 0$. In our example, since there is an edge from node $v_0$ to node $v_4$, SDC introduces the difference constraint $s_0 - s_4 \leq 0$, ensuring that $v_4$ is scheduled no earlier than $v_0$. Similar constraints are formulated for other data-dependent edges. In this work, we leverage SDC formulation with a new technique that implements a fully differentiable SDC to handle dependency constraints in scheduling.

\paragraph{Constraint Programming (CP) --} CP is a paradigm for solving combinatorial problems and is an effective method for addressing scheduling problems by allowing both discrete variables and non-linear constraints \cite{laborie2018ibm}. Unlike LP, which focuses on optimizing a linear objective function and requires constraints to be linear, CP is based on feasibility (finding a feasible solution) rather than optimization (finding an optimal solution). It focuses on the constraints and variables rather than the objective function, which leads to its superiority in managing complex and logical constraints. This makes it ideal for loosely constrained discrete sequencing problems with disjunctive constraints. For example, CP is used to solve the problem of execution time minimization of compute graphs subject to a memory budget~\cite{laborie2018ibm,bartan2023moccasin}. However, while CP provides significant flexibility and powerful constraint satisfaction capabilities, it can also face challenges with scalability and efficiency.

\paragraph{Learning-based Scheduling --} 

ML approaches have been explored for combinatorial scheduling, particularly in compiler optimization and hardware synthesis, to enhance the Pareto frontier of runtime and quality. Topoformer~\cite{gagrani2022neural} introduces a novel attention-based graph neural network architecture for topological ordering, focusing on learning embeddings for graph nodes. While Topoformer has provided significant insights and demonstrated potential in leveraging ML for scheduling, its generalizability and scalability heavily depend on the availability and volume of data. 
Conversely, reinforcement learning (RL) with graph learning-based schedulers~\cite{chen2019RLS,yin2023respect,yin2023accelerating} aims to improve scalability and generalizability by learning from action rewards, thus eliminating the need for extensive data collection and model generalization required by supervised learning. However, these RL-based approaches still face challenges related to problem scalability, generalizability, and substantial runtime overhead in training.

\paragraph{Heuristic scheduling algorithms} Heuristic scheduling algorithms \cite{ahn2020ordering,graham1969bounds,paulin1989force,blum2003metaheuristics} play a critical role in scheduling as well. Notable examples include list scheduling \cite{graham1969bounds}, a greedy algorithm that prioritizes tasks based on a predefined order, and force-directed scheduling \cite{paulin1989force}, which aims to balance tasks and resources iteratively to achieve latency-constrained, minimum-resource scheduling. In addition, stochastic heuristic methods such as evolutionary algorithms \cite{blum2003metaheuristics,wall1996genetic} and simulated annealing \cite{van1992job} are particularly effective in escaping local optima in complex scheduling spaces. While heuristic approaches mostly focus on finding feasible solutions at low runtime costs, they often fall short of reaching the optimal solution.


\subsection{Gumbel-Softmax} 

Gumbel-Softmax is a continuous distribution on the simplex which can be used to approximate discrete samples \cite{maddison2016concrete,jang2016categorical,gumbel1954statistical}. With Gumbel-Softmax, discrete samples can be differentiable and their parameter gradients can be easily computed with standard backpropagation. 
Let $z$ be the discrete sample with one-hot representation with $k$ dimensions and its class probabilities are defined as $p_1, p_2, ..., p_k$. Then, according to the Gumbel-Max trick proposed by \cite{gumbel1954statistical}, the discrete sample $z$ can be presented by:
\vspace{-1mm}
\begin{equation}
    z = \texttt{one\_hot}(\underset{i}{\text{argmax}}[g_i + log{p_i}])
    \label{equ:z}
\end{equation}
where $g_i$ are i.i.d samples drawn from Gumbel(0, 1). Then, we can use the differentiable approximation \texttt{Softmax} to approximate the one-hot representation for $z$, i.e., $\nabla_{p} z \approx \nabla_{p} y$:
\begin{equation}
    y_i = \frac{\text{exp}((\text{log}(p_i) + g_i)/\tau)}{\sum_{i=1}^{k} \text{exp}((\text{log}(p_i) + g_i)/\tau)}
    \label{equ:y_i}
\end{equation}
where $i = 1, 2, ..., k$. The softmax temperature $\tau$ is introduced to modify the distributions over discrete levels. Softmax distributions will become more discrete and identical to one-hot encoded discrete distribution as $\tau \rightarrow 0$, while at higher temperatures, the distribution becomes more uniform as $\tau \rightarrow \infty$ \cite{jang2016categorical}. Gumbel-Softmax distributions have a well-defined gradient $\frac{\partial{y}}{\partial{p}}$ w.r.t the class probability $p$. When we replace discrete levels with Gumbel-Softmax distribution depending on its class probability, we are able to use backpropagation to compute gradients. 

Gumbel-Softmax provides solutions for the differentiation of discrete scheduling and discrete design space explorations in neural architecture search and quantization tasks~\cite{wu2019fbnet,wu2018mixed,he2020milenas,fu2021a3c,fu2021auto,baevski2019vq}.
{For instance, \cite{wang2023cardinality} leverages the Gumbel trick as well as Sinkhorn iterations for combinatorial optimization and utilizes Sinkhorn to implement the problem constraints.}
However, the study of constrained discrete search and optimization through sampling methods, such as the Gumbel-Softmax, has not been extensively explored, which is particularly critical in scheduling and many other combinatorial optimization problems.

\section{Approach}
\label{sec:approach}

We propose a novel differentiable approach 
that \textit{compactly encodes} our targeted scheduling problem defined in Section \ref{subsec:schedule}, which can also be applied to a variety of important scheduling problems on dataflow graphs. Specifically, our method is capable of modeling (1) scheduling constraints in the SDC form and (2) an objective function for resource/cost minimization, both in a differentiable manner. 
We further introduce a novel \textit{constrained Gumbel Trick}, enabling highly parallelizable scheduling optimization through a sampling-based process with gradient descent. 

The remainder of this section will describe the formulation of the targeted scheduling problem in the SDC form, detailing the (1) definition of the search space, (2) modeling of dependencies, and (3) cost metrics (optimization objectives). Afterward, we will present our differentiable approach, aligning it with these three key components.


\subsection{Differentiable SDC}\label{sec:ilp2}
With our latency-constrained min-resource scheduling problem, we intend to schedule the node set $V$ on $L$ scheduling stages while minimizing the cost objectives defined in Section \ref{sec:diff_cost}. 
Note that we allow node chaining, which means two dependent nodes can be scheduled in the same stage, but a node cannot be scheduled earlier than its predecessor.

As mentioned earlier, a critical aspect of scheduling is honoring the dependencies between nodes, which can be specified using SDC. Specifically, these dependencies are translated into integer linear inequalities, which ensure that the resulting schedule adheres to the necessary precedence and resource constraints, maintaining the integrity of the data flow. Specifically, the inequality constraints can be summarized w.r.t the edges $E$, 
\begin{equation}  
    \forall e(i,j) \in E: s_i - s_j \leq c_{ij} \quad
    \label{eq:sdc_k_2_t}
\end{equation}
\noindent where \( e(i,j) \) denotes an edge that connects node \( i \) to node \( j \). The term \( s_i \) and \( s_j \) are the 
schedule variables for nodes \( i \) and \( j \), respectively. Given that we operate under a latency constraint $L$, all schedule variables follow constraint $\leq L$. To fully encode the scheduling problem as a differentiable model, our approach first addresses the vectorization of the search space, i.e., the vectorization of SDC variables, and then handles the integer inequality constraints with differentiable modeling. 

\subsubsection{Search Space Vectorization}

Given latency constraint $L$, a vector \(p\) in \(\mathbb{R}^L\) represents the probability vector of the scheduling decision \(\overrightarrow{s}\) for a given node, and the sampled decision \(\overrightarrow{s}\) is generated via hard Gumbel-Softmax \(\overrightarrow{s}\) = GS(\(p\)). As  Equation \ref{equ:z} indicates, \(\overrightarrow{s}\) is a one-hot vector, which contains a single `1' in its \(t^{th}\) coordinate and zeros elsewhere, {indicating the variable is scheduled at the scheduling stage $t$,} while \(p\) represents the probability distribution of $t$ falling into $[0, L-1]$. 
Therefore, for any scheduled variable vectorized as \(p\) in \(\mathbb{R}^L\), its corresponding integer solution space can be defined as $t$ \(\in [0, L-1]\) and \(L\) is the latency upper bound. 

Therefore, in the context of an SDC-encoded schedule, the solution values for each variable can be defined within its vector representation, i.e., \(\overrightarrow{s_i} \in \mathbb{R}^L\), with \(\text{argmax}(\overrightarrow{s_i}) \in [0, L-1]\). The search space can be fully vectorized by defining all the schedule variables {in SDC forms and capturing their dependencies in SDC}. 
Given a DAG \(G(V, E)\), our differentiable approach will first define the schedule variables in vector representation, for all \(\overrightarrow{s_i} \in \mathbb{R}^L\), \(i \in V\). 
This vectorization establishes a bijection between each integer value and its corresponding one-hot vector. Considering all schedule variables, the search space can be effectively represented by the tensor product of these one-hot vectors. As a result, the total optimizable parameters are in \(\mathbb{R}^{|V| \times L}\).

\subsubsection{Differentiable Modeling of Inequality Constraints}

The next critical step is to ensure the dependency constraints are met using the proposed approach.  
Specifically, our differentiable scheduling aims to incorporate the dependency constraints defined in 
$E$ as input, following the integer inequalities constraints in SDC, shown in Equation \ref{eq:sdc_k_2_t}. 
To encode these inequalities in a differentiable manner, we utilize the cumulative sum (\texttt{cumsum}) function. For a schedule variable represented as a one-hot vector \( \overrightarrow{s_i} \), the transformation using \texttt{cumsum} yields \( \texttt{cumsum}(\overrightarrow{s_i}) \), converting \( \overrightarrow{s_i} \) into its cumulative sum representation. Generally, the cumulative sum of a vector \( v = [v_0, v_1, ..., v_n] \) is \( v' = [v_0, v_0 + v_1, ..., \sum_{i=0}^{n} v_i] \). For one-hot vectors, this transformation indicates the feasible space for subsequent Gumbel-softmax sampling operations.

Given an integer inequality constraint \(s_i - s_j \leq c_{ij}\), 
we express it in vector form as \( \overrightarrow{s_i} \xrightarrow{c_{ij}} \overrightarrow{s_j} \), where \(\overrightarrow{s_i}\) is the sampled discrete solution of \(s_i\), and we define a transformation \(T_{\leq}: \mathbb{R}^n \times \mathbb{R} \to \mathbb{R}^n\) for the "\(\leq\)" constraint, where $\hat{+}$ is an operator that rolls the '1' to the right in the one-hot vector by $|c_{ij}|$ position(s):  
\begin{equation}  
\mathbf{T_{\leq}}(\overrightarrow{s_i}, c_{ij}) = \texttt{cumsum}(\overrightarrow{s_i} \hat{+} {|c_{ij}|})
\end{equation}
This transformation \(\mathbf{T_{\leq}}(\overrightarrow{s_i}, c_{ij})\) effectively constrains the solution space for \(s_j\), represented as \(\overrightarrow{s_j}\).  Note that non-zero $c_{ij}$ can be used to capture additional timing constraints. 

\noindent
{\textbf{Example} -- We illustrate the differentiable integer inequality modeling using the constraint $s_0 - s_1 \leq 0$, with $s_0, s_1 \in \mathbb{R}^3$. Let the initial sampling of $\overrightarrow{s_0}$ be \([0,1,0]\), such that $s_0$ evaluates to `1`. 
Then, $T_{\leq}$ evaluates to \([0,1,1]\), which implies that if $s_0=1$, $s_1$ can only be sampled as 1 or 2, which can be confirmed by the original integer inequality constraint. 
If additional timing constraint is given, e.g., $c_{ij}=-1$, $T_{\leq}$ will be evaluated to \([0,0,1]\) by rolling $T_{\leq}$ to the right by one position, which results in $s_1=2$ to satisfies the additional timing constraint.}

Therefore, we introduce the \textit{constrained Gumbel Trick}, enabling our model to handle inequality constraints. We use \(T_{\leq}\) in conjunction with Gumbel distribution sampling to ensure that the sampling always satisfies the constraints:
\begin{equation}  
\label{eq:gumel_trick}
y'_i=\frac{\exp((\log(p_i)+g_i)/\tau)}{\sum_j \exp((\log(p_j)+g_j)/\tau)} \cdot \mathbf{T_{\leq}}, \quad g \sim \text{Gumbel}(0,1)
\end{equation}

\begin{lemma}
For the inequalities \( s_i - s_j \leq c_{ij} \), the transformation \( T_{\leq} \) ensures any sampled vector space for \( s_j \) satisfies the inequality.
\end{lemma}

\begin{proof}
Consider any arbitrary constraint \( s_i - s_j \leq c_{ij} \). The transformation \( T_{\leq} \) applied to \(\overrightarrow{s_i}\) restricts the feasible vector space for \(\overrightarrow{s_j}\). Any vector \( \overrightarrow{s_j} \) sampled from this space will satisfy the inequality when converted back to integer values via the bijection.
\end{proof}

\subsection{Optimization Process and Objectives}\label{sec:diff_cost}

Targeting latency-constrained min-resource scheduling, the optimization process aims to search for the best possible scheduling solutions with the given latency and dependency constraints, guided by a loss function that minimizes resource costs.

\subsubsection{optimization process}
As discussed in Section \ref{sec:ilp2}, the latency constraint is modeled via our vectorization process of schedule variables, and the dependency constraints are ensured via differentiable inequality modeling. The core of this process is the constrained Gumbel Trick, where we employ vectorized representations of scheduling decisions. Firstly, we calculate the logits for each scheduling decision, then we incorporate these logits into the GS 
function with our constraint transformations to obtain the sampling probabilities. The process can be mathematically formulated as follows:

Let \(\overrightarrow{s}\) be the vector representation of a schedule variable. We make use of the GS function with the constrained Gumbel Trick in Equation \ref{eq:gumel_trick} to obtain the sampling probability:
\begin{equation}  
P = \text{GS}(\mathbf{T_{\leq}}(\overrightarrow{s}, {c}); \tau),
\end{equation}
where \( P \in \mathbb{R}^L \) and \( P^i \in [0, 1] \) for each \( i \in [0, L-1] \) with \( L \) being the scheduling depth upper bound. Here, \( \mathbf{T_{\leq}} \) represents the transformation for inequality constraints and \( \text{GS}(\cdot; \tau) \) is the GS function with temperature \( \tau \).

The probability of selecting a scheduling solution can be calculated by considering both the conditional probabilities under the constraints and the overall probability of the solution being feasible. For a scheduling solution \( k \), the probability \( P(i) \) for a node \( i \) can be computed as:
\begin{align}
P(i) &= P(i|k) \cdot P(k) \\
     &= P(i|k) \cdot p(cl(i))
\end{align}
where \( P(i|k) \) is the conditional probability of choosing node \( i \) given the scheduling solution \( k \), and \( p(cl(i)) \) represents the probability of the scheduling class \( cl(i) \) being feasible under the given constraints.

\begin{figure*}
    \centering
    \includegraphics[width=0.95\linewidth]{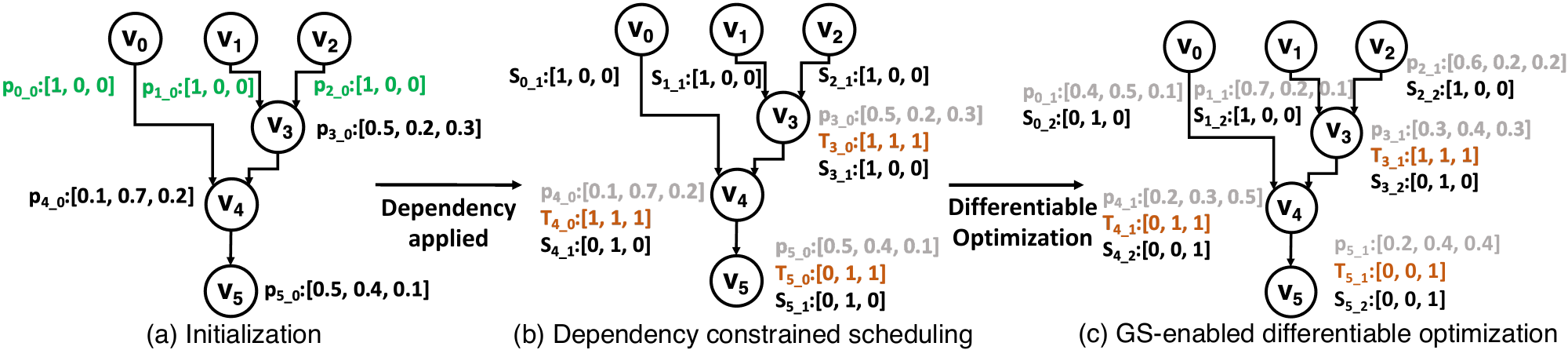}
    \caption{The implementation overview of our differentiable scheduling. $p$ indicates the probability vector for each node during GS-enabled differentiable  optimization; $T$ indicates the dependency constraints from the predecessors 
    $S$ indicates the scheduled stage for the node. (a) The search space vectorization and GS initialization. (b) The legal scheduling after applying the dependency constraints. (c) The scheduling is optimized with GS-enabled differentiable optimization.
    }
    \label{fig:framework_overview}
\end{figure*}

\subsubsection{Differentiable Cost Models}

Finally, we introduce a differentiable loss function that integrates the target objectives. As discussed in Section \ref{subsec:schedule}, we target the scheduling problem of minimizing two cost objectives associated with
the nodes and edges of the graph in the form of weights: 1) maximum memory resource utilization calculated with weights of the nodes, and 2) cross-stage communication cost with the weights of the edges. 

Specifically, we illustrate the two targeted optimization objectives and metrics using the example in Figure~\ref{fig:sdc-example}. Considering a schedule where $v_0$, $v_1$, and $v_2$ execute in the first stage, while $v_5$ executes in the last stage, the communication cost is then calculated as the sum of all data transferred between stages $I_0$ and $I_1$. As for the memory cost, peak memory refers to the maximum memory used across all stages.  Note that the cost metrics can be assessed uniquely w.r.t a given schedule.

{To enable parallelizable optimization using gradient descent} w.r.t the target objectives, we integrate a differentiable cost function {$\mathcal{L}$ based on the scheduling result.} 
To minimize memory usage under a latency constraint, we define the memory loss function \(\mathcal{L}_{e}\), which includes the entropy of scheduled nodes over \(L\) stages:
\begin{equation} 
    \mathcal{L}_{e} = -\sum_{i=0}^{L-1} \frac{N_{i}}{M}\log \frac{N_{i}}{M}
    \label{equ:loss_entropy}
\end{equation}
where \(N_{i}\) represents the memory of all nodes at the \(i\)-th stage, and \(M\) is the total memory of all nodes. Assuming uniform memory requirements for each node, \(N_{i}\) is equivalent to the number of nodes at the \(i\)-th stage, and \(M = |V|\), the total node count. Minimizing \(\mathcal{L}_{e}\) aims to evenly distribute the required memory across stages, which correlates to minimizing the peak memory resource cost. The effectiveness of this entropy-based approach for scheduling resource minimization is originally proven in~\cite{wang2010multi}.

Furthermore, to account for the minimization of communication cost, we add $\mathcal{L}_{c}$ into the loss functions. In this context, we simply formulate the communication cost as the mean of total cost over all the inter-stages, i.e.,
\begin{equation}
    \mathcal{L}_{c} = \frac{1}{\sum_{b=0}^{|E|-1}c_{b}} \sum_{i=0}^{L-2} m_i
\end{equation}
\noindent where $m_i$ is the accumulated communication cost for all edges on each inter-stage $i$, $c_{b}$ is the communication cost introduced by edge $e_b$. The final loss is then defined as
\begin{equation}  
    \mathcal{L} = \lambda \mathcal{L}_{e}+\mathcal{L}_{c}
\end{equation}
\noindent where $\lambda$ is a customizable input used to adjust the optimization ratio between the two objectives.

\paragraph{Implementation} The overview of the differentiable scheduling optimization is shown in Figure \ref{fig:framework_overview}. First, each schedule variable (per node) is initialized with the vectorized search space. More concretely, given 3 available stages, we use 3-dimensional vectors that are made differentiable with GS 
(Figure \ref{fig:framework_overview}(a)). Note that, the primary input nodes 
($v_{0}$, $v_{1}$ and $v_{2}$) are initialized with the distribution biased to the first stages in GS. 
Then, the constraints $T$ are applied to the initialized $p$ to produce the legal one-hot scheduling $S$. For example, for $v_{5}$, even though the first stage shows the highest probability in $p_{5\_0}$, its feasible search space is constrained by its predecessors 
$v_{4}$ and $v_{3}$ with $T_{5\_0}=[0, 1, 1]$ and the legal scheduling is then selected within the feasible space with the highest probability, i.e., the second stage. In each optimization iteration, in the forward path, the loss function is computed by the legal one-hot scheduling $S$ to guarantee its legalization, and the scheduling is optimized with the backward propagation through the probability vector $p$ in GS to make it differentiable. By updating the probability vector $p$, the one-hot vector \(\overrightarrow{s}\) for variable scheduling given by \(\overrightarrow{s}\) = GS(\(p\)) will be updated accordingly in the next iteration to realize the differentiable iterative optimization. 

\begin{algorithm}[!thb]
   \caption{Differentiable  Scheduling}
   \label{alg:training}
\begin{algorithmic}[1]
   \REQUIRE Graph $G(V,E)$; Targeted latency $L$
   \FOR{each $n \in V$}
   \STATE $W_n \leftarrow$ Initialize($L$)
   \STATE $W \leftarrow W \cup \{W_n\}$
   \ENDFOR
   \FOR{$i=1$  to \textit{num\_epochs}}
   \STATE $S^i \leftarrow [\;]$
   \FOR{each node $t$ in TopologicalOrder($V$)}
   \STATE $GS \leftarrow$ GumbelSoftmax($W_t^i$)
   \STATE $GS_{E_t} \leftarrow$ ApplyConstraints($GS, E_t$)
   \STATE $S_t^i \leftarrow$ ExtractOneHot($GS_{E_t}$)
   \STATE $S^i \leftarrow S^i \cup \{S_t^i\}$
   \ENDFOR
   \STATE $\mathcal{L} \leftarrow$ $\mathcal{L}$($S^i$); $\nabla W^i \leftarrow \nabla \mathcal{L}(W^i)$
   \STATE $W^{i+1} \leftarrow$ UpdateParameters($W^i, \nabla W^i, \eta$)
   \ENDFOR
\end{algorithmic}
\end{algorithm}

 We illustrate the training process of our approach in Algorithm \ref{alg:training} for a given graph $G(V,E)$ with a latency bound $L$. 
Initially, the weights $W_n$ are initialized for each node $n \in V$ and collectively stored in $W$. The algorithm proceeds for a specified number of epochs. In each epoch, an empty schedule $S^i$ is created. The nodes are processed in a topological order. For each node $t$, a Gumbel-Softmax distribution is computed using the current weights $W_t^i$. {Constraints specific to the edges $E_t$ which start from node $t$ are applied to this distribution by \textit{constrained Gumbel Trick}, and a one-hot encoded vector is extracted to represent the schedule of node $t$, which is then added to the schedule $S^i$.} The loss $\mathcal{L}$ is computed based on the current schedule $S^i$, and the gradient $\nabla W^i$ of the loss w.r.t. the weights is calculated. Finally, the weights are updated using the computed gradient and a learning rate $\eta$, resulting in updated weights $W^{i+1}$ for the next epoch. {After training all epochs, a final schedule $S$ is output by the algorithm. }

\begin{figure*}[htbp]
\centering

\begin{subfigure}[b]{0.475\textwidth}
    \includegraphics[width=0.95\textwidth]{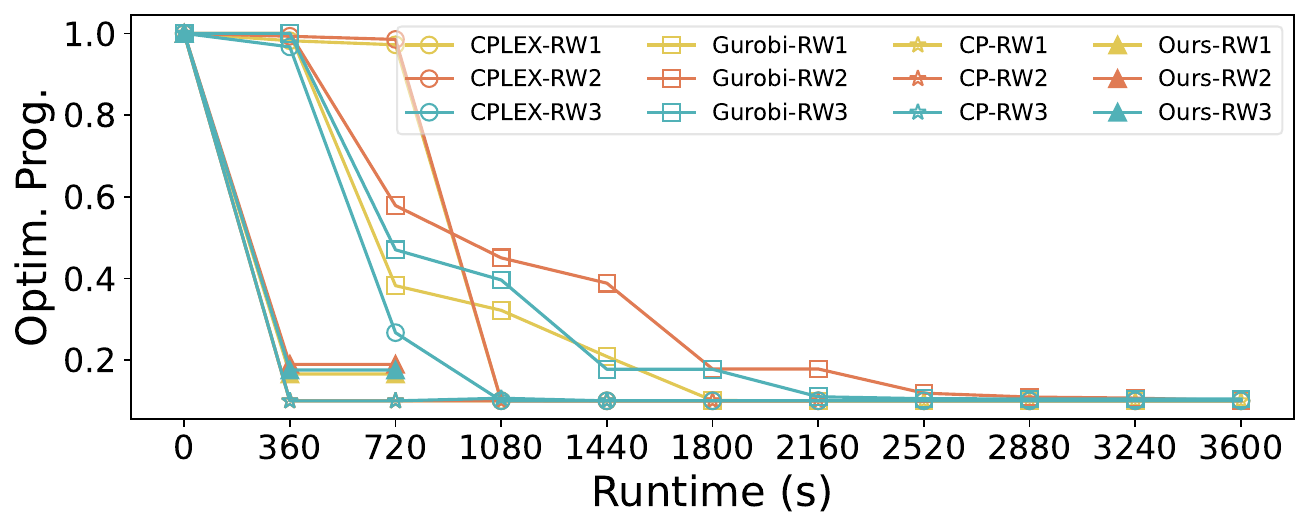}
    \caption{RW 1--3}
    \label{fig:random_1-3}
\end{subfigure}
\begin{subfigure}[b]{0.475\textwidth}
    \includegraphics[width=0.95\textwidth]{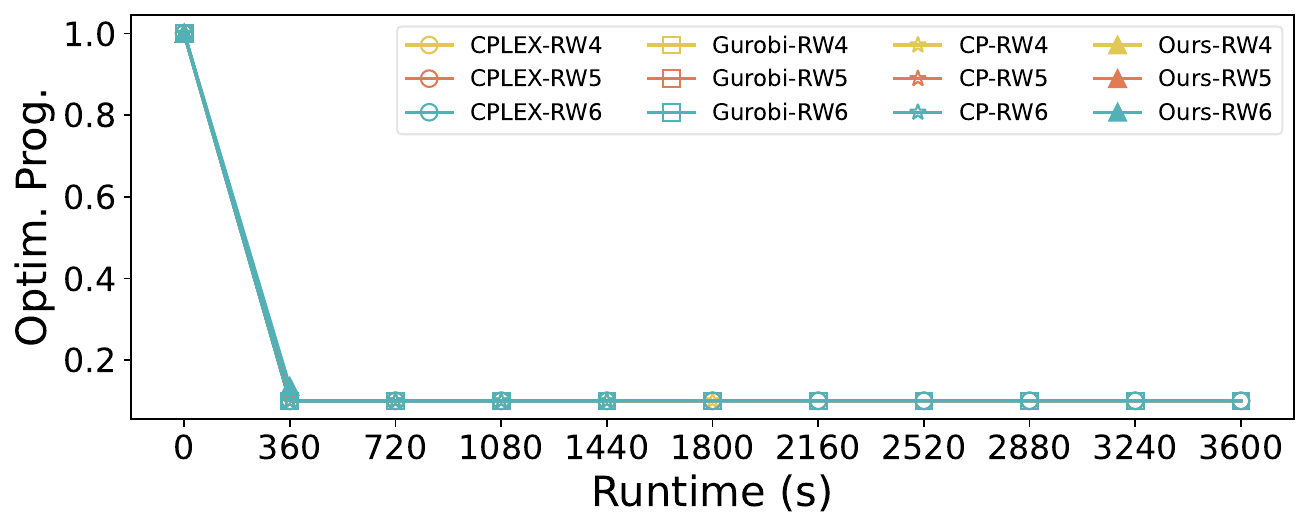}
    \caption{RW 4--6}
    \label{fig:random_4-6}
\end{subfigure}

\begin{subfigure}[b]{0.475\textwidth}
    \includegraphics[width=0.95\textwidth]{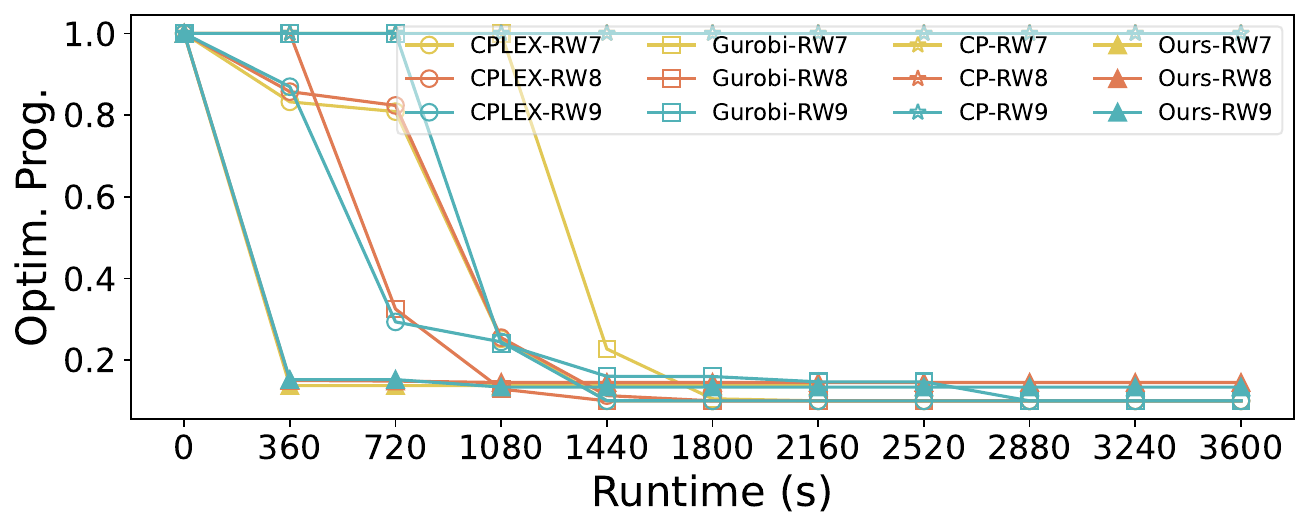}
    \caption{RW 7--9}
    \label{fig:random_7-9}
\end{subfigure}
\begin{subfigure}[b]{0.475\textwidth}
    \includegraphics[width=0.95\textwidth]{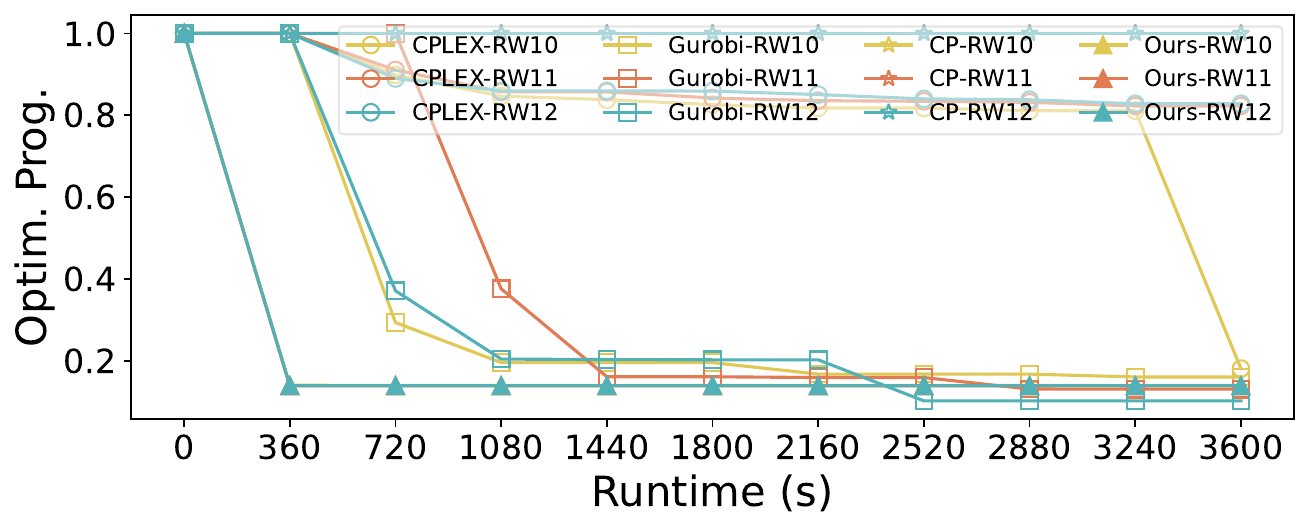}
    \caption{RW 10--12}
    \label{fig:random_10-12}
\end{subfigure}
\caption{Performance comparisons with random workloads. Baseline results are SDC scheduling solved by commercial SOTA CPLEX, Gurobi, and CP-SAT solvers.}
\label{fig:random_results}
\end{figure*}

\section{Experiment}
\label{sec:experiment}

\textbf{Benchmarks and baselines} Our experiments are conducted on specialized scheduling problems in GPU-based circuit simulation \cite{zhang2022gatspi}. Due to the nature of Boolean circuits and GPU runtime, the efficiency of simulation workload runtime is highly correlated with data transfer and GPU memory usage. Specifically, \cite{zhang2022gatspi} pointed out that the critical challenge in scheduling these workloads is to minimize communication overhead and avoid memory bandwidth bottlenecks. 
While the proposed approach can be evaluated with classic high-level synthesis benchmarks (e.g., MachSuite \cite{reagen2014machsuite} and Rosseta \cite{zhou2018rosetta}) or neural network computation graphs \cite{yin2023accelerating,yin2023respect,steiner2022olla}, we did not explicitly evaluate those benchmarks due to the fact they are trivial in reaching the optimal using the existing SOTA solvers\footnote{See our discussion on limitations in Section \ref{sec:limitation}.}. Our GPU workloads/graphs are derived using six designs from the EPFL Benchmark Suite~\cite{amaru2015epfl}, alongside baseline SDC+LP formulation solved by the SOTA commercial solvers, CPLEX~\cite{cplex2009v12} and Gurobi~\cite{gurobi}, as well as open-source tool CP-SAT solver~\cite{cpsatlp}. Note that these are non-traditional GPU workloads in scheduling. 
We extend the benchmarks by adding derived GPU computational graphs of simulated technology-mapped designs using 7nm ASAP technology library~\cite{xu2017standard}.  
Additionally, we add twelve synthetic random workloads (RW), all summarized in Table \ref{tab:design-spec}. We predefined the ratio of LP and the factor of our method, setting $\mathcal{R} = \lambda = 100$ for all EPFL designs and $\mathcal{R} = \lambda = 10$ for all synthetic workloads, respectively. We set the targeted latency to be $L=10$ for the experiments. All experiments were conducted using an an Intel\textregistered Xeon\textregistered Gold 6418H CPU and NVIDIA RTX\texttrademark 4090 GPU. A timeout of 3600 seconds was enforced for all designs across all methods.



\begin{table}[ht]
\centering
\scriptsize
\caption{Number of nodes and depths for selected designs. \textit{*RW}: Random workload. (M): Mapped.}
\label{tab:design-spec}
\begin{tabular}{|c|cc||c|cc|}
\hline
\textbf{Design} & \textbf{$|V|$} & \textbf{\#Depth} & \textbf{Design} & \textbf{$|V|$} & \textbf{\#Depth} \\ 
\hline
Adder & 1661 & 258 & RW1 & 949 & 15 \\ 
\hline
Adder (M) & 1830 & 89 & RW2 & 941 & 16 \\ 
\hline
Barrel shifter & 3734 & 15 & RW3 & 929 & 16 \\ 
\hline
Barrel shifter (M) & 2265 & 10 & RW4 & 810 & 9 \\ 
\hline
i2c controller & 1793 & 23 & RW5 & 819 & 8 \\ 
\hline
i2c controller (M) & 1221 & 11 & RW6 & 829 & 8 \\ 
\hline
Max & 4019 & 290 & RW7 & 4087 & 10 \\ 
\hline
Max (M) & 3493 & 97 & RW8 & 4063 & 9 \\ 
\hline
Square & 18742 & 253 & RW9 & 4086 & 8 \\ 
\hline
Square (M) & 18389 & 96 & RW10 & 8058 & 9 \\ 
\hline
Voter & 15761 & 73 & RW11 & 8192 & 11 \\ 
\hline
Voter (M) & 20058 & 39 & RW12 & 8193 & 9 \\ 
\hline
\end{tabular}
\end{table}

\subsection{Performance Comparison}

\begin{figure*}[t]
\centering

\begin{subfigure}[b]{0.475\textwidth}
    \includegraphics[width=0.95\textwidth]{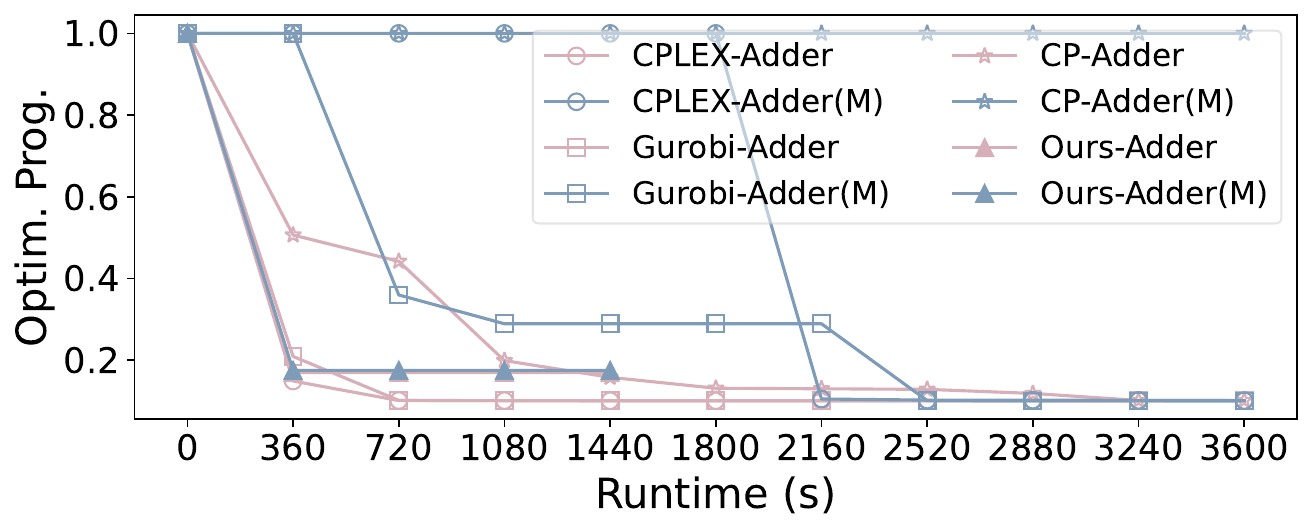}
    \caption{Adder}
    \label{fig:epfl_figs_adder}
\end{subfigure}
\begin{subfigure}[b]{0.475\textwidth}
    \includegraphics[width=0.95\textwidth]{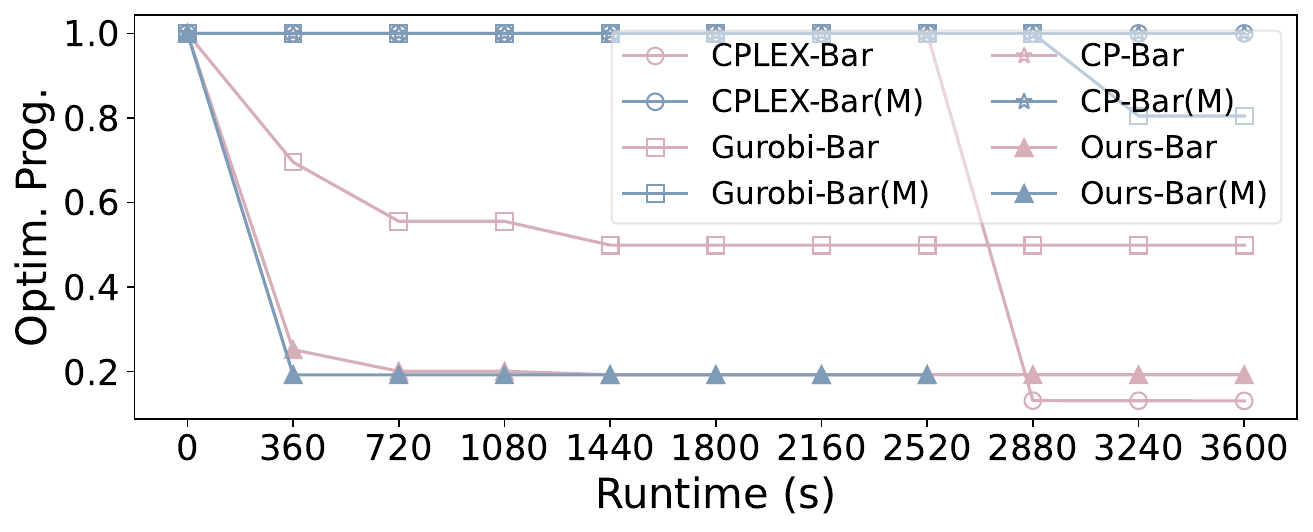}
    \caption{Barrel shifter}
    \label{fig:epfl_figs_bar}
\end{subfigure}

\begin{subfigure}[b]{0.475\textwidth}
    \includegraphics[width=0.95\textwidth]{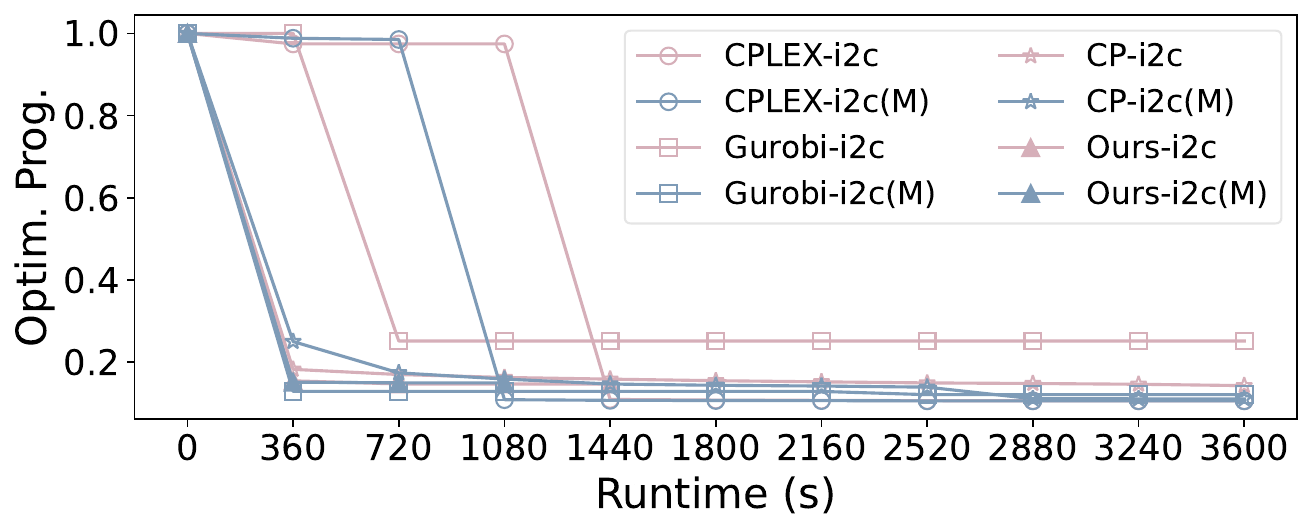}
    \caption{i2c controller}
    \label{fig:epfl_figs_i2c}
\end{subfigure}
\begin{subfigure}[b]{0.475\textwidth}
    \includegraphics[width=0.95\textwidth]{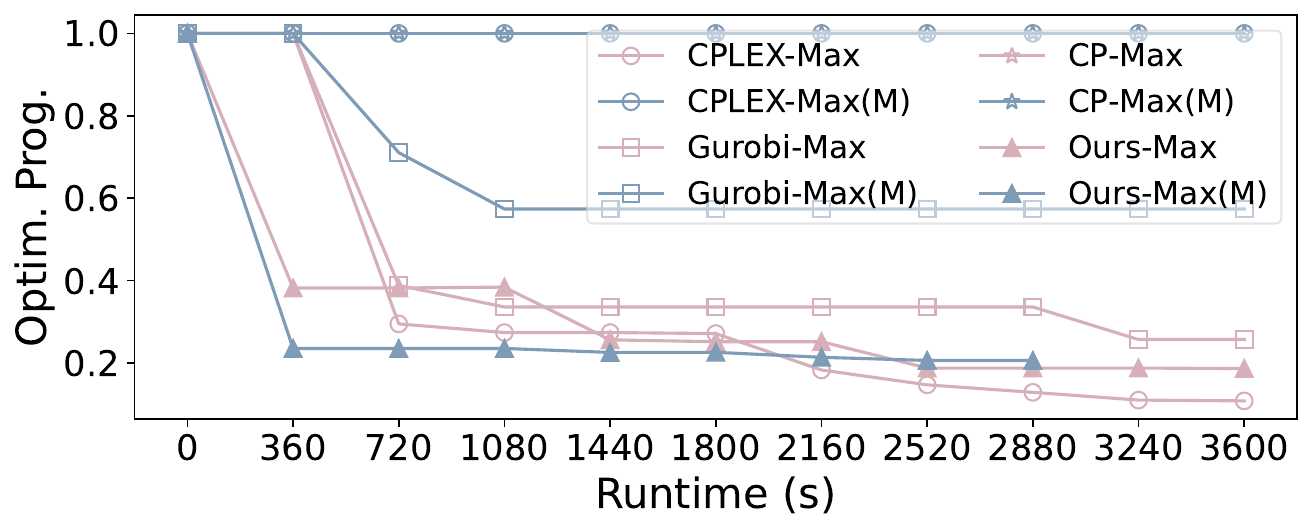}
    \caption{Max}
    \label{fig:epfl_figs_max}
\end{subfigure}
\begin{subfigure}[b]{0.475\textwidth}
    \includegraphics[width=0.95\textwidth]{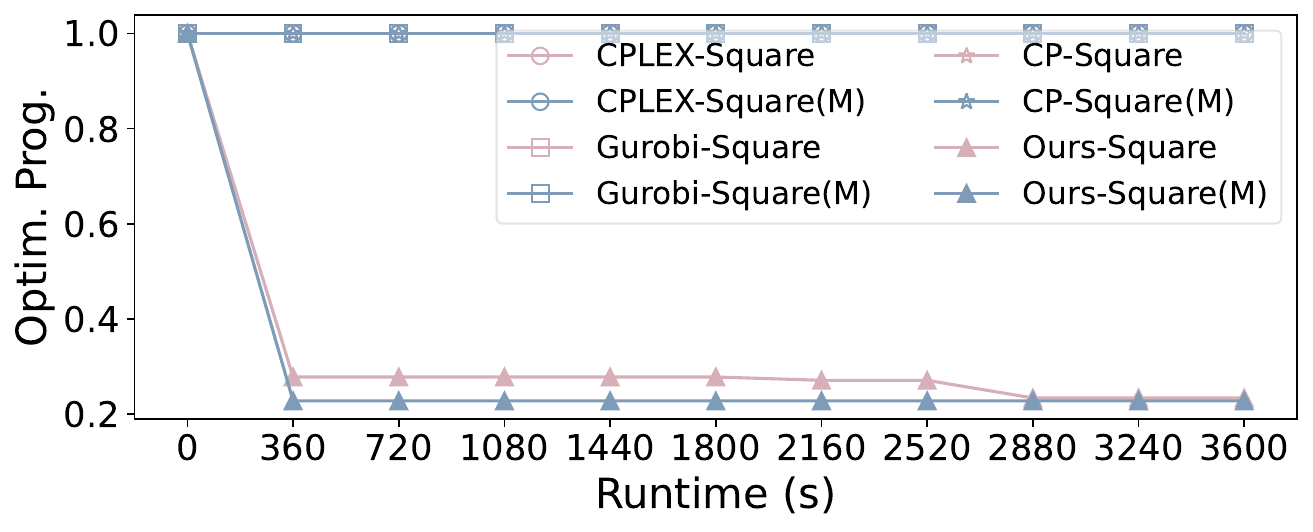}
    \caption{Square}
    \label{fig:epfl_figs_square}
\end{subfigure}
\begin{subfigure}[b]{0.475\textwidth}
    \includegraphics[width=0.95\textwidth]{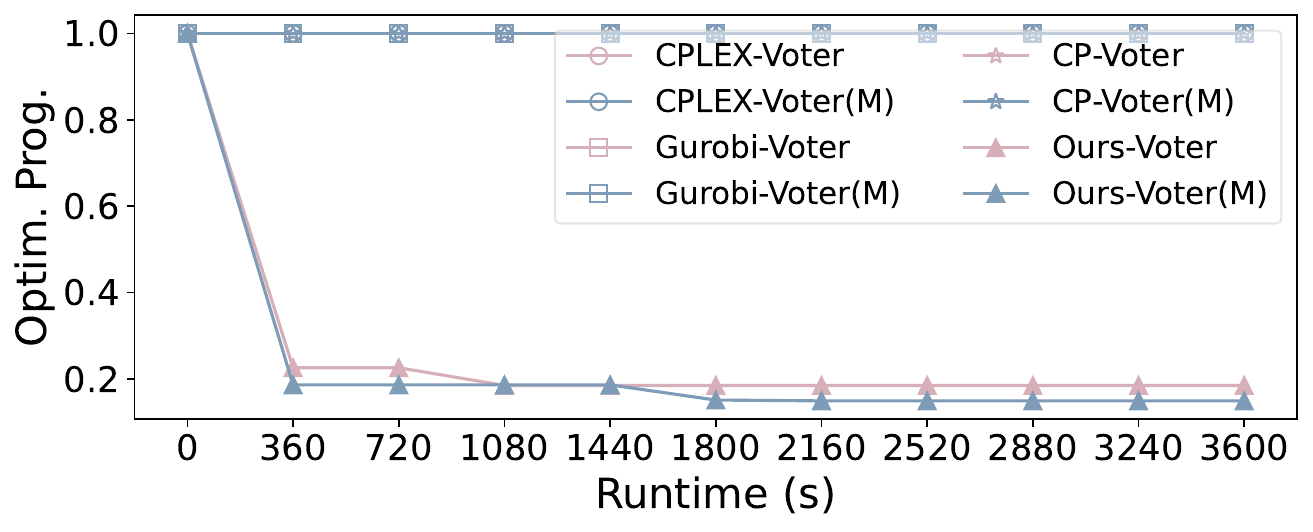}
    \caption{Voter}
    \label{fig:epfl_figs_voter}
\end{subfigure}
\caption{Performance comparisons with workloads built for EPFL benchmarks. Baseline results are SDC scheduling solved by commercial SOTA CPLEX, Gurobi, and CP-SAT solvers.}
\label{fig:epfl_figs}
\end{figure*}

\begin{figure*}[htbp]
\centering
\begin{subfigure}[b]{0.33\textwidth}
    \includegraphics[width=0.95\textwidth]{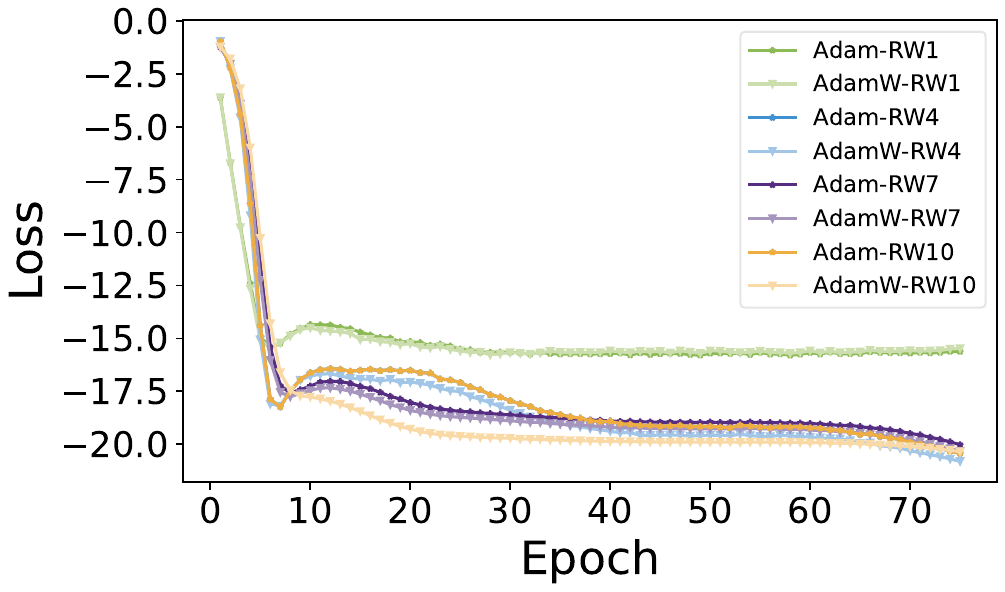}
    \caption{Selected Random workloads Loss}
    \label{fig:loss_figs_rand}
\end{subfigure}
\begin{subfigure}[b]{0.33\textwidth}
    \includegraphics[width=0.95\textwidth]{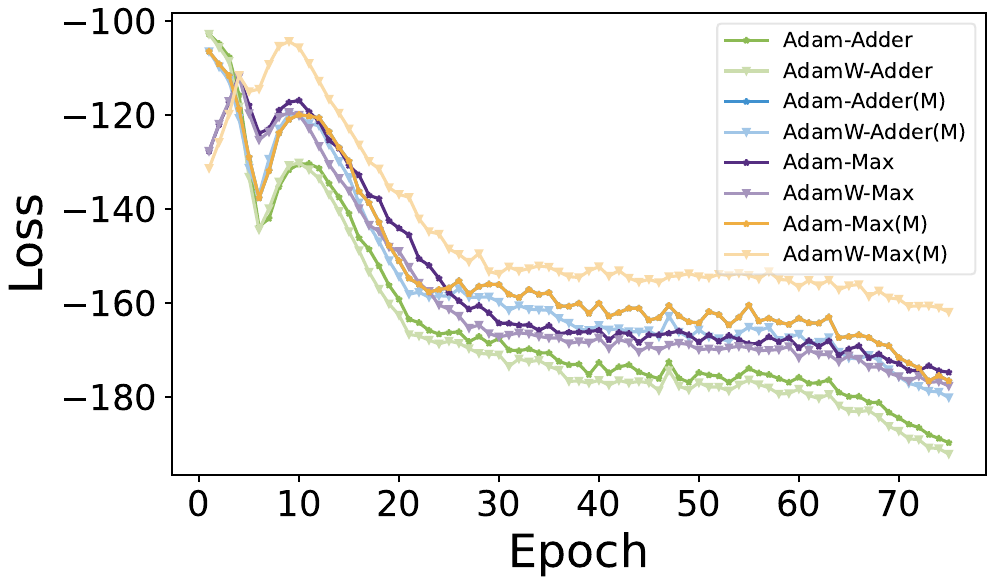}
    \caption{Selected EPFL designs Loss}
    \label{fig:loss_figs_epfl1}
\end{subfigure}
\begin{subfigure}[b]{0.33\textwidth}
    \includegraphics[width=0.95\textwidth]{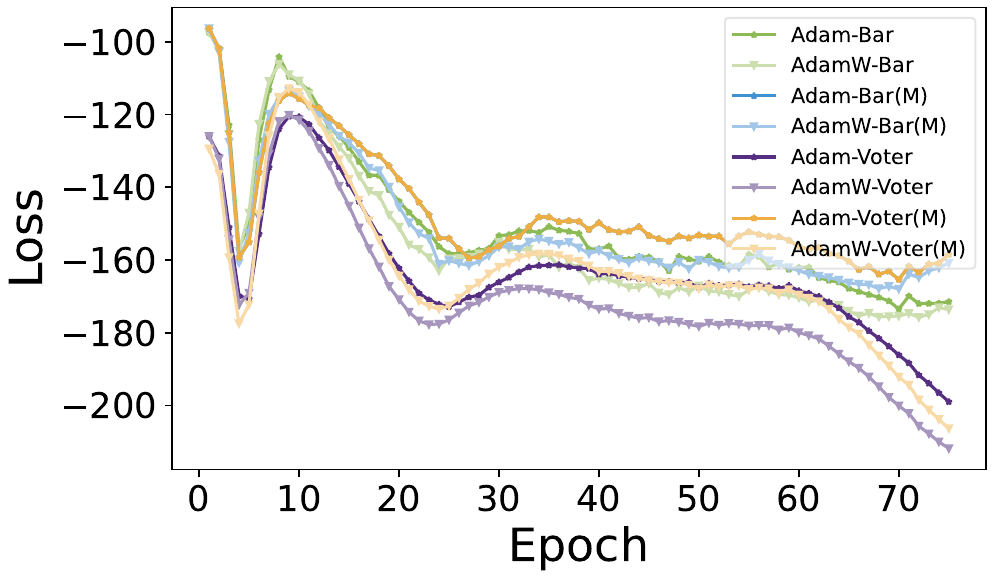}
    \caption{Selected EPFL designs Loss}
    \label{fig:loss_figs_epfl2}
\end{subfigure}
\caption{Loss Function Evaluation with selected Random workloads and EPFL designs.}
\label{fig:loss_figs}
\end{figure*}

We established 11 result sampling points at intervals of 360 seconds, ranging from 0 to 3600 seconds, to implement a consistent timeout across all designs. To ensure a fair comparison among the three methods, we employed optimization progress as a normalized factor. Initially, the objective value for all methods is set to $1.0$. During optimization, as the tentative objective value decreases, we capture the current best objective value at each sampling point. This value is then expressed as a ratio relative to the initial point, providing a normalized measure of progress. 
\paragraph{Evaluation on Synthetic benchmarks} 
As illustrated in Figure~\ref{fig:random_results}, we grouped three workloads with similar setups into one subfigure to highlight the similarity in optimization behavior. We observe that our method is particularly effective in optimizing the scheduling solution at early stages, regardless of graph size and density. For instance, as shown in Figures~\ref{fig:random_1-3} and \ref{fig:random_4-6}, our method converges within the initial 360 seconds for both denser workloads (RW 1-3) and less dense workloads (RW 4-6). In contrast, CPLEX and Gurobi barely initiate the optimization process for denser workloads and only achieve comparable performance for sparser workloads. Additionally, Figures~\ref{fig:random_7-9} and \ref{fig:random_10-12} show that our method maintains a similar convergence speed within the first 360 seconds, whereas CPLEX and Gurobi experience nearly a 3$\times$ decrease in performance when the workload scale increases from 5000 to 10000. {Meanwhile, CP-SAT shows its superiority in Figure~\ref{fig:random_1-3}, but its performance degrades sharply and fails to produce any optimization when solving larger and denser workloads, as shown in Figure~\ref{fig:random_7-9} and \ref{fig:random_10-12}.} This demonstrates the robustness and stability of our method in scalability.

\paragraph{Evaluation on EPFL Benchmarks} 
Following the observations from the synthetic workloads, we extended our verification to real-world designs. Similar to our previous approach, we grouped the original design and its corresponding mapped design in a single subfigure, due to their shared graph characteristics. Consistent with our initial observations, our method consistently exhibits a similar convergence trend between the original and mapped designs, across a range of graph densities. For instance, as demonstrated in Figure~\ref{fig:epfl_figs_bar}, while CPLEX and Gurobi manage to perform adequately on the original (sparse) designs, they exhibit nearly a 2$\times$ performance degradation on the mapped (dense) designs during most stages of the total timeout period. Furthermore, with very large designs such as those depicted in Figures~\ref{fig:epfl_figs_square} and \ref{fig:epfl_figs_voter} for Square and Voter respectively, CPLEX and Gurobi barely initiate optimization within the 3600 seconds timeout, whereas our method achieves convergence to desirable points early on (within a 360 seconds timeout). {As for CP-SAT, it only achieves comparable performance on very limited designs, as shown in Figure 4a and 4c, which are both represented as graphs with smaller and sparser scales. However, it fails to initiate the optimization for the rest of the tested designs.} This underscores the advantages of our method in handling large and complex scheduling cases. 


\subsection{Loss Function Evaluation}



Considering that our method demonstrates consistent early convergence across a diverse range of workloads and designs, we focus our investigation on the initial 75 epochs to illustrate the effectiveness of the formulated loss function. Given the similarities among the workloads, only a representative subset has been selected to accommodate page constraints while ensuring a comprehensive analysis.

As depicted in Figure~\ref{fig:loss_figs_rand}, for the selected synthetic workloads, fast convergence of the loss function is observed during the first 30 epochs, independent of graph size. This rapid convergence indicates the robustness of our method across varying data sizes. Similarly, for the selected EPFL designs, as illustrated in Figures~\ref{fig:loss_figs_epfl1} and \ref{fig:loss_figs_epfl2}, the convergence pattern remains consistent, showcasing rapid convergence within the first 75 epochs. An initial instability observed during the first 10 epochs can be attributed to the necessity of a warm-up phase for the Gumbel-Softmax mechanism employed in discrete sampling by our method. This brief period of instability is typical and expected as the model adjusts to the discrete nature of the data.
Furthermore, we explored the impact of different weight decay settings by transitioning from the Adam optimizer to AdamW. Upon examining Figure~\ref{fig:loss_figs}, no consistent advantage is observed in the loss convergence trends between the two optimizers. This observation suggests that the optimal choice of weight decay may be design-specific, indicating further investigation. 

\subsection{Limitations and Discussion}
\label{sec:limitation}

While our method exhibits rapid convergence across a variety of designs, our analysis has unveiled certain limitations for future directions. First, we note that in some specific cases, the optimization quality achieved by our approach falls short when compared to the benchmarks set by established solvers such as CPLEX, Gurobi, and CP-SAT. This discrepancy is particularly evident in designs of a smaller scale or those characterized by sparse graph structures. This observation suggests that our method might benefit from enhanced strategies tailored to these specific problem characteristics, possibly through refined optimization techniques or algorithmic adjustments that better leverage the properties of sparsity and scale. 

Second, we observe that while the problem is less complex (trivial), i.e., the problems can be solved very effectively by solving SDC+LP models using LP/CP solvers, our approach does not offer much advantage. Furthermore, we empirically observe that the complexity of the targeted problem is highly associated with the density of the graph $G(V,E)$. As the density of $G(V,E)$ increases, our proposed approach will advance further in the runtime-quality frontier.

Thirdly, our empirical findings indicate instances of 'overfitting' in the loss convergence beyond certain epochs, as illustrated in Figure \ref{fig:loss_figs_epfl2}, which could limit the ability to optimize the objectives. 
Addressing this challenge might involve exploring advanced regularization methods or adaptive stopping criteria that preemptively halt training before overfitting occurs. Furthermore, these limitations not only highlight areas for improvement but also underscore the importance of developing more nuanced optimization methods capable of adapting to varying problem scales and complexities. Future work could focus on devising novel approaches that explicitly account for the structural characteristics of the problem domain, thereby enhancing the robustness and effectiveness of the method across a broader spectrum of scheduling application scenarios.

\vspace{-1.5mm}
\section{Conclusion}
\label{sec:conclusion}
 In this work, we propose an end-to-end differentiable formulation for combinatorial and scalable scheduling, utilizing model-free dataless auto-differentiation with customized objective functions. By harnessing GPU capabilities, our experimental results demonstrate substantial performance improvements over SOTA methods solved with commercial and open-source solvers such as CPLEX, Gurobi, and CP-SAT. Our limitations point towards the necessity for further refinement, particularly in achieving more competitive optimization outcomes and addressing overfitting issues in loss convergence. It potentially broadens the applicability and efficiency of our method in solving complex scheduling problems and other discrete combinatorial problems.

\newpage
\typeout{} 
\section*{Acknowledgement}
This work is supported in part by National Science Foundation (NSF) awards \#2047176, \#2019306, \#2019336, \#2008144, \#2229562, \#2403134, \#2403135, and ACE, one of the seven centers in JUMP 2.0,  a Semiconductor Research Corporation (SRC) program sponsored by DARPA.

\section*{Impact Statement}
This paper aims to advance the field of Machine Learning through the development of a novel differentiable combinatorial scheduling algorithm. Our work has many potential societal implications, including enhanced efficiency over industrial SOTA methods. Finally, this work does not raise any ethical concerns that need to be addressed at this time.

\bibliography{icml2024/ref.bib}
\bibliographystyle{icml2024}


\end{document}